\documentclass{article}

\usepackage{microtype}
\usepackage{graphicx}
\usepackage{subfigure}
\usepackage{subcaption}
\usepackage{booktabs} 

\usepackage{hyperref}

\usepackage[preprint]{icml2026}

\usepackage{amsmath}
\usepackage{amssymb}
\usepackage{mathtools}
\usepackage{amsthm}

\usepackage[capitalize,noabbrev]{cleveref}

\theoremstyle{plain}
\newtheorem{theorem}{Theorem}[section]
\newtheorem{proposition}[theorem]{Proposition}

\theoremstyle{definition}
\newtheorem{definition}[theorem]{Definition}

\theoremstyle{remark}

\usepackage[textsize=tiny]{todonotes}

\usepackage{listings}
\usepackage{xcolor}
\definecolor{codegray}{rgb}{0.5,0.5,0.5}
\definecolor{codepurple}{rgb}{0.58,0,0.82}
\definecolor{backcolour}{rgb}{0.95,0.95,0.95}

\lstdefinestyle{mystyle}{
    backgroundcolor=\color{backcolour},
    commentstyle=\color{gray},
    keywordstyle=\color{blue},
    numberstyle=\tiny\color{codegray},
    stringstyle=\color{codepurple},
    basicstyle=\ttfamily\footnotesize,
    breakatwhitespace=false,
    breaklines=true,
    captionpos=b,
    keepspaces=true,
    numbers=left,
    numbersep=5pt,
    showspaces=false,
    showstringspaces=false,
    showtabs=false,
    tabsize=2
}
\icmltitlerunning{Submission and Formatting Instructions for ICML 2026}

\begin{document}

\twocolumn[
  \icmltitle{Dynamic Vocabulary Pruning: Stable LLM-RL by Taming the Tail}



  \icmlsetsymbol{equal}{*}
  \begin{icmlauthorlist}
    \icmlauthor{Yingru Li}{}
    \icmlauthor{Jiawei Xu}{cuhksz}
    \icmlauthor{Jiacai Liu}{fdu}
    \icmlauthor{Yuxuan Tong}{}
    \icmlauthor{Ziniu Li}{cuhksz}
    \icmlauthor{Tianle Cai}{}
    \icmlauthor{Ge Zhang}{}
    \icmlauthor{Qian Liu}{}
    \icmlauthor{Baoxiang Wang}{cuhksz}
  \end{icmlauthorlist}

  \icmlaffiliation{cuhksz}{The Chinese University of Hong Kong, Shenzhen}
  \icmlaffiliation{fdu}{Fudan University}


  \icmlcorrespondingauthor{Yingru Li}{szrlee@gmail.com}

  \icmlkeywords{Machine Learning, ICML}

  \vskip 0.3in
]



\printAffiliationsAndNotice{}  

\begin{abstract}
Reinforcement Learning (RL) for Large Language Models (LLMs) faces a fundamental tension: the numerical divergence between high-throughput inference engines and numerically precise training engines. Although these systems share the same parameters, they produce slightly different probability distributions, creating a training-inference mismatch. We prove that the bound on the log-probability divergence arising from this mismatch scales as $(1-p)$, where $p$ is the token probability. This scaling induces a highly asymmetric effect: the bound vanishes for high-probability tokens but remains significant for low-probability tokens in the distribution tail. When sampled, these tail tokens introduce systematically biased errors that accumulate over sequences, thereby destabilizing gradient estimation. Instead of applying post-hoc corrections, we propose Dynamic Vocabulary Pruning (DVP), which constrains the RL objective to a dynamically determined ``safe'' vocabulary that excludes the extreme tail. This strategy trades large, destabilizing numerical errors for a small, bounded optimization bias. We validate DVP empirically by demonstrating stable training, and theoretically by deriving strict bounds on the induced bias.
\end{abstract}

\section{Introduction}

Reinforcement Learning (RL) has emerged as a key technique for training Large Language Models (LLMs) on complex reasoning~\cite{shao2024deepseekmath,zeng2025simplerl} and multi-turn agentic tasks~\cite{xue2025simpletir,jin2025search}, relying on outcome-based rewards to drive performance. However, scaling RL to modern LLMs encounters a severe computational bottleneck: \emph{rollout generation}. Estimating policy gradients requires high-throughput trajectory sampling, necessitating the use of specialized inference engines (e.g., vLLM~\cite{kwon2023efficient}, SGLang~\cite{zheng2024sglang}). These engines achieve speed through aggressive optimizations—such as paged attention, low-precision KV-caches (INT8/FP8), and fused CUDA kernels. Conversely, training frameworks (e.g., Megatron-LM \cite{shoeybi2019megatron}, PyTorch FSDP \cite{zhao2023pytorch}) prioritize numerical stability and gradient precision, operating at higher precision (FP32 or mixed precision).

This dichotomy creates a \textbf{training-inference mismatch}: the inference policy $\pi_{\text{infer}}$, optimized for throughput, deviates subtly but systematically from the training policy $\pi_{\text{train}}$, used for gradient computation. In practice, this divergence frequently precipitates training collapse~\cite{liu-li-2025-rl-collapse,yao2025offpolicy}. While enforcing bitwise equivalence between inference and training systems would resolve this discrepancy, it would catastrophically reduce throughput, negating the benefits of modern inference. Thus, we face a fundamental speed-versus-consistency tradeoff. As inference engines pursue more aggressive approximations to reduce latency, this gap widens, rendering training instability not a transient bug, but a persistent structural challenge in the LLM-RL.

In this work, we view training instability not as a technical bug requiring reactive correction, but as a symptom of a poorly-specified learning objective. Standard RL objectives demand accurate gradient estimation over the entire vocabulary, including the long tail of low-probability tokens. We identify that this is mathematically fragile: the bound on log-probability mismatch does not contract for low-probability tokens as it does for high-probability ones. Consequently, minor numerical deviations in the ``tail'' of the distribution result in disproportionately large gradient errors.

Based on this insight, we propose \textbf{Dynamic Vocabulary Pruning (DVP)}. Rather than relying on post-hoc clipping or ad-hoc regularization, DVP redesigns the learning objective to operate exclusively over a dynamically determined ``safe'' vocabulary at each generation step. By systematically excluding the numerical noise of the probability tail, we recover stability without sacrificing inference speed.

Our specific contributions are as follows:

\begin{enumerate}
    \item \textbf{Rigorous Diagnosis:} We characterize the mathematical structure of off-policy gradient instability in LLMs. We prove the log-probability mismatch bound scales as $(1-p)$, where $p$ is the token probability. This vulnerability is asymmetric: vanishing for high-probability tokens but remaining significant for low-probability tokens. Furthermore, we show that these tail errors are not zero-mean noise but possess systematic bias that accumulates over long horizons (Section~\ref{sec:diagnosis}).
    
    \item \textbf{Principled Framework:} We introduce Dynamic Vocabulary Pruning (DVP), a method that utilizes min-$p$ filtering~\cite{nguyen2024turning} to construct constrained policies. This approach addresses the root cause of instability at the objective level, ensuring that gradients are computed only where the training and inference policies are numerically consistent (Section~\ref{sec:framework}).
    
    \item \textbf{Empirical Validation:} We demonstrate that DVP effectively mitigates training collapse. Our experiments on mathematical reasoning tasks show that DVP enables stable training and performance improvements over standard baselines (Section~\ref{sec:experiments}).
\end{enumerate}

\section{Related Work}
\label{sec:related_work}

RL has become the standard paradigm for aligning LLMs with complex objectives~\cite{ouyang2022training, bai2022constitutional}. Recent advances focus on improving reasoning capabilities through outcome-based supervision rather than dense human annotation. Several methods~\cite{gulcehre2023reinforced,zelikman2022star} utilize iterative self-improvement, where the model generates its own training data filtered by ground-truth correctness. More advanced algorithms such as PPO~\cite{schulman2017proximal} and GRPO~\cite{shao2024deepseekmath} have been adapted to optimize reasoning chains directly. However, these methods typically assume consistency between the inference policy and the training policy. Our work addresses the instability that arises when this assumption is violated by the implementation discrepancy.

To address the latency of autoregressive generation, significant effort has been directed toward system-level optimizations. Techniques range from memory-efficient attention mechanisms like PagedAttention in vLLM~\cite{kwon2023efficient} and SGLang~\cite{zheng2024sglang}, to low-precision quantization methods (INT8/FP8)~\cite{dettmers2022gpt3, xiao2023smoothquant} and speculative decoding~\cite{leviathan2023fast}. While these optimizations dramatically increase throughput, they introduce numerical deviations from the high-precision computations used during training (e.g., Pytorch FSDP~\cite{zhao2023pytorch}, Megatron-LM~\cite{shoeybi2019megatron}). Prior work~\cite{liu-li-2025-rl-collapse,yao2025offpolicy} has largely treated these deviations as negligible implementation details; we rigorously characterize them as a source of structural instability in RL training.

Controlling the LLM vocabulary output space has been explored for both acceleration and safety. Static pruning methods remove tokens based on frequency or linguistic rules, while dynamic methods like min-$p$ sampling~\cite{nguyen2024turning}, nucleus sampling (top-$p$)~\cite{holtzman2019curious}, and top-$k$ limit the sampling pool to high-probability tokens to prevent degeneration. In the context of RL, entropy regularization~\cite{ahmed2019understanding} and KL-divergence penalties~\cite{jaques2019way} are standard tools to keep the policy close to a reference. However, these soft constraints do not strictly prevent the evaluation of gradients on the unstable tail of the distribution. Our proposed Dynamic Vocabulary Pruning (DVP) integrates sampling-based filtering directly into the learning objective to ensure numerical stability.

\section{Preliminaries}
\label{sec:preliminaries}

We formalize autoregressive text generation as a Markov Decision Process (MDP). At each timestep $t$, the state $s_t = [x; y_1, \ldots, y_{t-1}]$ consists of the initial prompt $x$ concatenated with the sequence of tokens generated thus far. The action $a_t$ corresponds to selecting a token from the vocabulary $\mathcal{V}$, sampled according to the parameterized policy $\pi_\theta(\cdot|s_t)$. The cumulative generation process yields a complete trajectory $y = (y_1, \ldots, y_T)$. Upon completion, the trajectory receives a sparse reward $R(x, y) \in \{0, 1\}$, typically indicating the correctness of the reasoning trace or the successful completion of a multi-turn agentic task.

The  RL objective is to maximize the expected reward over the distribution of trajectories induced by the policy:
\begin{equation}
    J(\theta) = \mathbb{E}_{y \sim \pi_\theta}[R(x, y)].
\end{equation}
We optimize this objective via the policy gradient algorithms. By applying the log-derivative trick and the chain rule over the sequence length $T$, the gradient takes the form:
\begin{equation}
    \nabla_\theta J(\theta) = \mathbb{E}_{y \sim \pi_\theta}\left[ R(x, y) \sum_{t=1}^{T} \nabla_\theta \log \pi_\theta(y_t | s_t) \right].
\end{equation}
This formulation assumes that trajectories are sampled from the same policy $\pi_\theta$ used for gradient computation. However, a critical challenge arises in the off-policy setting created by the training-inference mismatch, where trajectories are sampled from an inference policy $\pi_{\text{infer}}$ while gradients are computed with respect to the training policy $\pi_{\text{train}}$.

\section{Diagnosis: The Fundamental Instability}
\label{sec:diagnosis}

Training instability in LLM-RL is not merely an engineering nuisance but a phenomenon with deep mathematical roots. In this section, we diagnose this instability by rigorously analyzing the \emph{training-inference mismatch}. We demonstrate that small numerical deviations, inherent to high-throughput inference, translate into systematic biases that disproportionately affect the tail of the probability distribution.

\subsection{Gradient Bias from Mismatch}
We consider two policies that share parameters $\theta$ but differ in computational implementation:
\begin{itemize}
    \item $\pi_{\text{train}}$: The training policy, executed with high precision to ensure stable gradient updates.
    \item $\pi_{\text{infer}}$: The inference policy, executed with aggressive optimizations to maximize sampling throughput.
\end{itemize}

Ideally, the policy gradient should be estimated using samples drawn from $\pi_{\text{train}}$. However, system constraints force us to sample trajectories from $\pi_{\text{infer}}$ while computing gradients using $\pi_{\text{train}}$. This introduces a fundamental bias:

\begin{theorem}[Gradient Bias from Mismatch]
\label{thm:mismatch_bias}
Let $g = \mathbb{E}_{y \sim \pi_{\text{train}}}[\nabla_\theta \log \pi_{\text{train}}(y|x) \cdot R(x,y)]$ be the ideal gradient, and let $g' = \mathbb{E}_{y \sim \pi_{\text{infer}}}[\nabla_\theta \log \pi_{\text{train}}(y|x) \cdot R(x,y)]$ be the practical gradient. The bias $b = g' - g$ is given by:
\begin{equation}
    b = \mathbb{E}_{y \sim \pi_{\text{train}}} \left[ \left( e^{-\Delta_y} - 1 \right) \cdot \nabla_\theta \log \pi_{\text{train}}(y|x) \cdot R(x,y) \right],
\end{equation}
where $\Delta_y = \log \pi_{\text{train}}(y|x) - \log \pi_{\text{infer}}(y|x)$ represents the cumulative sequence-level log-probability mismatch.
\end{theorem}

\begin{proof}
We transform the practical gradient $g'$ using importance sampling to share the same expectation basis as the ideal gradient $g$:
\begin{align}
    g' &= \mathbb{E}_{y \sim \pi_{\text{infer}}}\left[\nabla_\theta \log \pi_{\text{train}}(y|x) \cdot R(x,y)\right] \\ \nonumber
       &= \mathbb{E}_{y \sim \pi_{\text{train}}} \left[ \frac{\pi_{\text{infer}}(y|x)}{\pi_{\text{train}}(y|x)} \cdot \nabla_\theta \log \pi_{\text{train}}(y|x) \cdot R(x,y) \right].
\end{align}
Substituting the mismatch $\Delta_y$, we note that:
\begin{align}
    \frac{\pi_{\text{infer}}(y|x)}{\pi_{\text{train}}(y|x)} &= \exp\left( \log \pi_{\text{infer}}(y|x) - \log \pi_{\text{train}}(y|x) \right) \\ \nonumber
    &= \exp(-\Delta_y) = e^{-\Delta_y}.
\end{align}
Thus, the estimator becomes:
\begin{equation}
    g' = \mathbb{E}_{y \sim \pi_{\text{train}}} \left[ e^{-\Delta_y} \cdot \nabla_\theta \log \pi_{\text{train}}(y|x) \cdot R(x,y) \right].
\end{equation}
Subtracting the ideal gradient $g$ (which corresponds to the term where the ratio is 1) yields the bias:
\begin{align}
    b &= g' - g \\ \nonumber
    &= \mathbb{E}_{y \sim \pi_{\text{train}}} \left[ \left(e^{-\Delta_y} - 1\right) \cdot \nabla_\theta \log \pi_{\text{train}}(y|x) \cdot R(x,y) \right].
\end{align}
\end{proof}

The magnitude of this bias is governed by the probability ratio $e^{-\Delta_y} = \frac{\pi_{\text{infer}}(y)}{\pi_{\text{train}}(y)}$. When $\Delta_y$ is negative—implying that the inference engine assigns higher probability to a trajectory than the training engine ($\pi_{\text{infer}} \gg \pi_{\text{train}}$)—this ratio can grow exponentially. While Importance Sampling (IS) could theoretically correct this distribution shift, it requires weighting samples by the inverse ratio $\frac{\pi_{\text{train}}}{\pi_{\text{infer}}}$. In the regimes where mismatch is most severe, these weights become vanishingly small, resulting in prohibitive variance. Thus, we face a dilemma: accept a biased gradient (unweighted) or an unstable one (weighted).

\subsection{Modeling the Mismatch: Logit Perturbations}
To understand \emph{why} $\Delta_y$ diverges, we must analyze the token-level mechanics of the training-inference mismatch. Even with identical parameters $\theta$, distinct hardware backends produce different logits due to the non-associativity of floating-point arithmetic: $(a \oplus b) \oplus c \neq a \oplus (b \oplus c)$ in finite precision~\cite{he2025nondeterminism}.

Modern inference engines (e.g., vLLM~\cite{kwon2023efficient}, SGLang~\cite{zheng2024sglang}) and training frameworks (e.g., Megatron-LM~\cite{shoeybi2019megatron}, Pytorch FSDP~\cite{zhao2023pytorch}) diverge in three critical areas:
\begin{itemize}
    \item \textbf{Attention Reduction:} Implementations like PagedAttention~\cite{kwon2023efficient} and FlashAttention2~\cite{dao2023flashattention} utilize different summation orders for the softmax denominator, altering rounding errors.
    \item \textbf{Numerical Precision:} Inference often relies on INT8/FP8 KV-cache quantization, whereas training maintains higher precision for stability.
    \item \textbf{Operator Fusion:} Distinct kernel boundaries lead to different intermediate truncation points.
\end{itemize}

We model the aggregate effect of these discrepancies as an additive perturbation on the logits $z$:
\begin{equation}
    z^{\text{infer}} = z^{\text{train}} + \epsilon, \quad \text{where } \epsilon = (\varepsilon_1, \ldots, \varepsilon_{|\mathcal{V}|}).
\end{equation}
Since these errors arise from bounded-precision arithmetic, we assume $|\varepsilon_k| \leq \epsilon_{\max}$ for some small constant $\epsilon_{\max}$.

\subsection{Asymmetric Vulnerability}
A key contribution of our diagnosis is the observation that the impact of uniform logit noise is \emph{not} uniform across the vocabulary. The resulting log-probability error is structurally asymmetric.

\begin{proposition}[Asymmetric Vulnerability]
\label{prop:asymmetric_vulnerability}
Under the logit perturbation model $z^{\text{infer}} = z^{\text{train}} + \epsilon$ with bounded noise $|\varepsilon_k| \leq \epsilon_{\max}$, the magnitude of the token-level log-probability mismatch satisfies:
\begin{equation}
    |\Delta_a| \leq 2\epsilon_{\max}(1 - p_a),
\end{equation}
where $p_a = \pi_{\text{train}}(a|s)$ is the probability assigned by the training policy.
\end{proposition}

\begin{proof}
Let $z' = z + \epsilon$ be the perturbed logits. Define the log-softmax function $f_a(z) = z_a - \log\sum_j e^{z_j}$. By the Mean Value Theorem, the mismatch $\Delta_a$ is:
\begin{equation}
    \Delta_a = f_a(z) - f_a(z') = -\nabla f_a(z_c) \cdot \epsilon,
\end{equation}
for some intermediate point $z_c$. The gradient of the log-softmax is given by $\frac{\partial f_a}{\partial z_k} = \delta_{ak} - p_k(z_c)$. Substituting this into the dot product:
\begin{equation}
    -\Delta_a = (1 - p_a(z_c))\varepsilon_a - \sum_{k \neq a} p_k(z_c)\varepsilon_k.
\end{equation}
Applying the triangle inequality and noting $|\varepsilon_k| \leq \epsilon_{\max}$:
\begin{align}
    |\Delta_a| &\leq \epsilon_{\max}\left[ |1-p_a(z_c)| + \sum_{k \neq a} |p_k(z_c)| \right] \\ \nonumber
    &= \epsilon_{\max}\left[ (1-p_a(z_c)) + (1-p_a(z_c)) \right].
\end{align}
Thus, $|\Delta_a| \leq 2\epsilon_{\max}(1 - p_a(z_c))$. As $\epsilon_{\max} \to 0$, $p_a(z_c) \to p_a$, recovering the bound.
\end{proof}

This bound reveals a critical dichotomy:
\begin{itemize}
    \item \textbf{High-Probability Regime ($p_a \to 1$):} The term $(1-p_a)$ approaches zero, suppressing the error. For the ``best'' tokens, the mismatch vanishes naturally.
    \item \textbf{Low-Probability Regime ($p_a \to 0$):} The term $(1-p_a)$ will approach 1, leaving the error bound at its maximum ($2\epsilon_{\max}$).
\end{itemize}

This asymmetry implies the ``tail'' of the distribution is numerically fragile. Furthermore, this noise is not zero-mean in its effect. We analyze the statistical behavior of the mismatch when tokens are sampled from the inference policy.

\begin{proposition}[Signature of Failure]
\label{prop:failure_signature}
Assume perturbations are i.i.d. with $\varepsilon_k \sim (0, \sigma^2)$. Given that an action $a$ is sampled from $\pi_{\text{infer}}$, the mode of the mismatch $\Delta'_a = -\Delta_a$ is approximately:
{\small
\begin{equation}
    \text{Mode}[\Delta'_a \mid a \sim \pi_{\text{infer}}] \approx \sigma^2 \left[ (1 - p_a)(1 - p'_a) + \sum_{k \neq a} p_k p'_k \right],
\end{equation}
}
where $p'_a = \pi_{\text{infer}}(a|s)$.
\end{proposition}

\begin{proof}
Let $E_a$ denote the event ``action $a$ is sampled from $\pi_{\text{infer}}$''. We seek the most probable perturbation vector $\epsilon$ given this event. Using Bayes' theorem with a Gaussian prior on perturbations:
\begin{align}
    \log P(\epsilon | E_a) &\propto \log \pi_{\text{infer}}(a|s) + \log P(\epsilon) \\ \nonumber
    &= (z_a + \varepsilon_a) - \log\sum_j e^{z_j + \varepsilon_j} - \frac{1}{2\sigma^2}\sum_j \varepsilon_j^2.
\end{align}
We take the derivative with respect to $\varepsilon_k$ and set it to zero to find the mode $\epsilon^*$:
\begin{align}
    \frac{\partial}{\partial \varepsilon_k} \log P(\epsilon | E_a) &= \delta_{ak} - p'_k - \frac{\varepsilon_k}{\sigma^2} = 0 \\ \nonumber \implies \quad \varepsilon_k^* &= \sigma^2(\delta_{ak} - p'_k).
\end{align}
We substitute these optimal perturbations into the first-order approximation of the mismatch, $\Delta'_a = -\Delta_a \approx (1-p_a)\varepsilon_a - \sum_{k \neq a} p_k \varepsilon_k$:
\begin{equation}
    \text{Mode}[\Delta'_a] \approx (1-p_a)\sigma^2(1-p'_a) - \sum_{k \neq a} p_k \sigma^2(0 - p'_k).
\end{equation}
Simplifying yields:
\begin{equation}
    \text{Mode}[\Delta'_a] \approx \sigma^2 \left[ (1-p_a)(1-p'_a) + \sum_{k \neq a} p_k p'_k \right].
\end{equation}
\end{proof}

For low-probability tokens, this mode is strictly positive. This indicates a systematic bias where $\pi_{\text{infer}}(a|s) > \pi_{\text{train}}(a|s)$ for sampled tail tokens. This theoretical prediction corroborates recent empirical findings~\cite{liu-li-2025-rl-collapse} that inference engines tend to overestimate the probability of rare tokens relative to high-precision training baselines.

\subsection{Summary}
Our diagnosis identifies a structural failure mode. First, the vocabulary tail is inherently vulnerable to numerical instability (Proposition~\ref{prop:asymmetric_vulnerability}). Second, when low-probability tokens are sampled, they exhibit a systematic bias where $\Delta_a$ is negative (Proposition~\ref{prop:failure_signature}). Finally, because token-level errors accumulate over the sequence ($\Delta_y = \sum \Delta_{a_t}$), trajectories containing multiple tail tokens suffer from exploded probability ratios. This necessitates a solution that fundamentally excludes the unstable tail from the learning objective.

\section{Solution: Dynamic Vocabulary Pruning}
\label{sec:framework}

To address the training-inference mismatch, we propose \textbf{Dynamic Vocabulary Pruning (DVP)}. Rather than applying reactive patches to unstable gradients, we redesign the learning objective. By constraining the policy to a dynamically determined ``safe'' vocabulary at each step, we ensure that gradient estimation occurs only over the numerical support where the training and inference policies are consistent.

\subsection{Dynamic Vocabulary Pruning via Min-$P$ Filtering}
We employ min-$p$ sampling~\cite{nguyen2024turning} to identify the safe vocabulary. Unlike top-$k$ or top-$p$ (nucleus) sampling, which depend on the cardinality or cumulative mass of the distribution, min-$p$ creates a threshold relative to the model's confidence.

\begin{definition}[Min-$P$ Safe Action Sets]
\label{def:safe_sets}
Given a pruning threshold $\rho \in (0, 1]$, the safe action sets for the training and inference policies are defined as:
\begin{align}
    \mathcal{V}_S(s) &= \left\{ a \in \mathcal{V} : \pi_{\text{train}}(a|s) \geq \rho \cdot \max_{k} \pi_{\text{train}}(k|s) \right\}, \\
    \mathcal{V}'_S(s) &= \left\{ a \in \mathcal{V} : \pi_{\text{infer}}(a|s) \geq \rho \cdot \max_{k} \pi_{\text{infer}}(k|s) \right\}.
\end{align}
\end{definition}
In practice, $\rho$ is set to a small value (e.g., $e^{-13} \approx 2.3 \times 10^{-6}$), retaining a broad set of plausible tokens while surgically removing the extremely low-probability tail where numerical instability is unbounded.

\subsection{Constrained Policies and Objective}
We define the constrained policies by renormalizing the probability mass over the safe sets.

\begin{definition}[Min-$P$ Constrained Policies]
\label{def:constrained_policies}
The constrained training policy $\pi_{\text{train}}^{\text{mp}}$ and inference policy $\pi_{\text{infer}}^{\text{mp}}$ are:
\begin{align}
    \pi_{\text{train}}^{\text{mp}}(a|s) &= \frac{\pi_{\text{train}}(a|s)}{Z_\theta(s)} \cdot \mathbb{I}[a \in \mathcal{V}_S(s)], \\
    \pi_{\text{infer}}^{\text{mp}}(a|s) &= \frac{\pi_{\text{infer}}(a|s)}{Z'_\theta(s)} \cdot \mathbb{I}[a \in \mathcal{V}'_S(s)],
\end{align}
where $Z_\theta(s) = \sum_{k \in \mathcal{V}_S(s)} \pi_{\text{train}}(k|s)$ is the normalization constant (coverage mass).
\end{definition}

Our proposed objective $J_{\text{mp}}(\theta) = \mathbb{E}_{y \sim \pi_{\text{train}}^{\text{mp}}}[R(x, y)]$ optimizes the expected reward over this stable distribution.

\subsection{The Stable Gradient Estimator}
We derive an off-policy gradient estimator for the constrained objective.

\begin{theorem}[Constrained Gradient Estimator]
\label{thm:stable_estimator}
When sampling trajectories $y \sim \pi_{\text{infer}}^{\text{mp}}$, an unbiased estimator of $\nabla_\theta J_{\text{mp}}(\theta)$ is:
\begin{equation}
    \hat{g}_{\text{mp}} = \frac{\pi_{\text{train}}^{\text{mp}}(y| x)}{\pi_{\text{infer}}^{\text{mp}}(y | x)} \cdot \nabla_\theta \log \pi_{\text{train}}^{\text{mp}}(y| x) \cdot R(x, y).
\end{equation}
\end{theorem}

This estimator is well-defined when sampled trajectory lies within the training safe set (i.e., $y_t \in \mathcal{V}_S(s_t)$ for all $t$). Since high-probability tokens have small mismatch $|\Delta_a|$ (Proposition~\ref{prop:asymmetric_vulnerability}), the safe sets $\mathcal{V}_S$ and $\mathcal{V}'_S$ overlap nearly perfectly for the tokens that matter. In the rare event $y_t \in \mathcal{V}'_S \setminus \mathcal{V}_S$, the ratio is zero, preventing gradient explosion.


\subsection{Theoretical Analysis}
We analyze the bias-variance tradeoff introduced by DVP. By modifying the learning objective, we accept a bounded optimization bias in exchange for a significant reduction in gradient variance.

\subsubsection{Contrastive Gradient Decomposition}
The stability of the DVP estimator stems from its implicit contrastive structure.

\begin{proposition}[Contrastive Gradient Form]
\label{prop:contrastive_form}
For any action $a \in \mathcal{V}_S(s)$, the gradient of the constrained log-probability decomposes as:
\begin{align}
    \nabla_\theta \log \pi_{\text{train}}^{\text{mp}}(a|s) = \nabla_\theta &\log \pi_{\text{train}}(a|s) \\ \nonumber
    &- \mathbb{E}_{k \sim \pi_{\text{train}}^{\text{mp}}}[\nabla_\theta \log \pi_{\text{train}}(k|s)].
\end{align}
\end{proposition}

\begin{proof}
Starting from the definition $\log \pi_{\text{train}}^{\text{mp}}(a|s) = \log \pi_{\text{train}}(a|s) - \log Z_\theta(s)$:
\begin{align}
    &\nabla_\theta \log \pi_{\text{train}}^{\text{mp}}(a|s) \\ \nonumber
    &= \nabla_\theta \log \pi_{\text{train}}(a|s) - \nabla_\theta \log \left( \sum_{k \in \mathcal{V}_S} \pi_{\text{train}}(k|s) \right) \\ \nonumber
    &= \nabla_\theta \log \pi_{\text{train}}(a|s) - \frac{\sum\limits_{k \in \mathcal{V}_S} \pi_{\text{train}}(k|s) \nabla_\theta \log \pi_{\text{train}}(k|s)}{Z_\theta(s)} \\ \nonumber
    &= \nabla_\theta \log \pi_{\text{train}}(a|s) - \mathbb{E}_{k \sim \pi_{\text{train}}^{\text{mp}}}[\nabla_\theta \log \pi_{\text{train}}(k|s)].
\end{align}
\end{proof}
Both standard and constrained gradients follow the form $\nabla_\theta z_a - b$. However, the constrained baseline $b$ is computed strictly over the safe set $\mathcal{V}_S$. This prevents the gradient from being polluted by the 100,000+ tail tokens where logits are numerically unstable.

\subsubsection{Bounded Optimization Bias}
Finally, we quantify the cost of this stabilization. The deviation between the constrained objective $J_{\text{mp}}$ and the original objective $J$ is bounded by the excluded probability mass.

\begin{proposition}[Bias Bound]
\label{prop:bias_bound}
The optimization bias satisfies:
\begin{equation}
    |J_{\text{mp}}(\theta) - J(\theta)| \leq R_{\max} \cdot T \cdot (1 - Z_{\min}),
\end{equation}
where $Z_{\min} = \min_s Z_\theta(s)$ is the minimum retained probability mass.
\end{proposition}

\begin{proof}
Using the total variation (TV) distance bound $|\mathbb{E}_P[f] - \mathbb{E}_Q[f]| \leq \|f\|_\infty \cdot \text{TV}(P, Q)$, we analyze the distance between the trajectory distributions. The divergence accumulates over time steps:
\begin{equation}
    \text{TV}(\pi_{\text{train}}^{\text{mp}}, \pi_{\text{train}}) \leq \sum_{t=1}^T \mathbb{E}_{s_t}[\text{TV}(\pi_{\text{train}}^{\text{mp}}(\cdot|s_t), \pi_{\text{train}}(\cdot|s_t))].
\end{equation}
For a single step, the TV distance is half the $L_1$ norm:
\begin{equation}
    \text{TV}(\pi_{\text{train}}^{\text{mp}}(\cdot|s), \pi_{\text{train}}(\cdot|s)) = \frac{1}{2}\sum_a |\pi_{\text{train}}^{\text{mp}}(a|s) - \pi_{\text{train}}(a|s)|.
\end{equation}
Splitting the sum into $a \in \mathcal{V}_S$ and $a \notin \mathcal{V}_S$:
\begin{align}
    \text{TV}&(\pi_{\text{train}}^{\text{mp}}(\cdot|s), \pi_{\text{train}}(\cdot|s)) \\ \nonumber
    &= \frac{1}{2} \left[ \sum_{a \in \mathcal{V}_S} \left(\frac{\pi_{\text{train}}(a)}{Z_\theta} - \pi_{\text{train}}(a)\right) + \sum_{a \notin \mathcal{V}_S} \pi_{\text{train}}(a) \right] \\ \nonumber
    &= \frac{1}{2} \left[ \left(\frac{1}{Z_\theta} - 1\right)Z_\theta + (1 - Z_\theta) \right] = 1 - Z_\theta(s).
\end{align}
Thus, the total bias is bounded by $R_{\max} \sum_t \mathbb{E}[1 - Z_\theta(s_t)]$, which simplifies to the proposition claim.
\end{proof}

With a typical threshold $\rho = e^{-13}$, the coverage $Z_\theta(s)$ is effectively 1.0 in nearly all contexts. Thus, DVP incurs negligible bias while completely eliminating the gradient noise from the vocabulary tail.

\subsection{Practical Implementation}
\label{sec:practical_implementation}

Implementing Dynamic Vocabulary Pruning requires addressing the non-differentiable nature of the set construction $\mathcal{V}_S(s)$. In practice, we treat the safe set as fixed during backpropagation (using \texttt{torch.no\_grad()}), a standard approximation that avoids differentiating through the discrete inclusion decision. We realize this constraint efficiently through logit masking.

\paragraph{Logit Masking Strategy.}
The constrained policy $\pi_{\text{train}}^{\text{mp}}$ is implemented by applying a mask directly in logit space. Let $z$ denote the original logits. We define the masked logits $z_{\text{mp}}$ as:
\begin{equation}
    z_{\text{mp}, k} = \begin{cases} z_k & \text{if } k \in \mathcal{V}_S(s) \\ -\infty & \text{otherwise} \end{cases}.
\end{equation}
In low-precision training (e.g., BF16), we substitute $-\infty$ with a large negative finite value (e.g., $-50.0$) to maintain numerical stability.

We rigorously verify that this masking strategy yields the correct policy and gradient updates.

\begin{proposition}[Masked Logit Correctness]
For all actions $a \in \mathcal{V}_S(s)$, the following hold:
\begin{enumerate}
    \item \textbf{Policy Equivalence:} $\text{softmax}(z_{\text{mp}})_a = \pi_{\text{train}}^{\text{mp}}(a|s)$.
    \item \textbf{Gradient Equivalence:} $\nabla_\theta \log(\text{softmax}(z_{\text{mp}})_a) = \nabla_\theta \log \pi_{\text{train}}^{\text{mp}}(a|s)$, under the fixed set assumption.
\end{enumerate}
\end{proposition}

\begin{proof}
\textbf{Part 1: Policy Equivalence.} Substituting the masked logits into the softmax definition:
\begin{align}
    \text{softmax}(z_{\text{mp}})_a &= \frac{e^{z_{\text{mp}, a}}}{\sum_j e^{z_{\text{mp}, j}}} \\ \nonumber
    &= \frac{e^{z_a}}{\sum_{k \in \mathcal{V}_S(s)} e^{z_k} + \sum_{l \notin \mathcal{V}_S(s)} e^{-\infty}} \\ \nonumber
    &= \frac{e^{z_a}}{\sum_{k \in \mathcal{V}_S(s)} e^{z_k}} \\ \nonumber
    &= \frac{\pi_{\text{train}}(a|s)}{Z_\theta(s)} = \pi_{\text{train}}^{\text{mp}}(a|s).
\end{align}

\textbf{Part 2: Gradient Equivalence.} The gradient of the standard log-softmax function for a logit vector $u$ is $\nabla_\theta \log(\text{softmax}(u)_a) = \nabla_\theta u_a - \mathbb{E}_{k \sim \text{softmax}(u)}[\nabla_\theta u_k]$. Applying this to $z_{\text{mp}}$:
\begin{align}
    \nabla_\theta &\log(\text{softmax}(z_{\text{mp}})_a) \\ \nonumber
    &= \nabla_\theta z_{\text{mp}, a} - \sum_{k} \text{softmax}(z_{\text{mp}})_k \nabla_\theta z_{\text{mp}, k}.
\end{align}
Since $z_{\text{mp}, k} = z_k$ for $k \in \mathcal{V}_S$ and the probability mass is zero otherwise, the expectation restricts to the safe set:
\begin{equation}
    = \nabla_\theta z_a - \sum_{k \in \mathcal{V}_S} \pi_{\text{train}}^{\text{mp}}(k|s) \nabla_\theta z_k = \nabla_\theta z_a - \mathbb{E}_{k \sim \pi_{\text{train}}^{\text{mp}}}[\nabla_\theta z_k].
\end{equation}
This recovers the contrastive gradient form derived in Proposition~\ref{prop:contrastive_form}.
\end{proof}

\paragraph{Hyperparameter Selection.}
The threshold $\rho$ governs the intrinsic bias-variance tradeoff of the method:
\begin{itemize}
    \item \textbf{High $\rho$ (e.g., $10^{-2}$):} Aggressive pruning significantly reduces gradient variance by enforcing a small safe set, but introduces higher optimization bias as valid reasoning tokens might be excluded.
    \item \textbf{Low $\rho$ (e.g., $10^{-20}$):} Lenient pruning retains a larger vocabulary, minimizing bias but potentially failing to filter the unstable tail tokens that cause divergence.
\end{itemize}
In our experiments, we found $\rho = e^{-13} \approx 2.26 \times 10^{-6}$ to be an optimal balance point. This value is sufficiently small to retain all linguistically and mathematically plausible tokens (preserving diversity) while strictly pruning the floating-point noise floor where the relative error $\frac{\delta}{p}$ explodes. We provide detailed implementation code in Appendix~\ref{app:implementation}.



\section{Experiments}
\label{sec:experiments}

We evaluate Dynamic Vocabulary Pruning on mathematical reasoning tasks, a domain where long-horizon generation makes training particularly susceptible to the accumulation of numerical errors. We conduct our experiments using the Qwen3-14B-Base model as the policy initialization and utilize the filtered DAPO dataset\footnote{\url{https://huggingface.co/datasets/Jiawei415/DPAO_filter/tree/main/train}} for training. We employ RLOO~\cite{ahmadian2024back} as the base algorithm with a group size of 16. To accommodate the deep reasoning chains required for these problems, we configure the environment with a maximum response length of 16,384 tokens. Training is performed in a full on-policy configuration, with both the rollout batch size and the mini-update batch size set to 32.

To rigorously benchmark our approach, we compare DVP against standard methods designed to handle policy mismatch. These include naive RLOO (which lacks explicit correction), Token-level Truncated Importance Sampling (TIS)~\cite{yao2025offpolicy}, and Masked Importance Sampling (MIS)~\cite{liu-li-2025-rl-collapse}. We set the threshold to 0.2 for both IS strategies. For our DVP implementation, we employ a conservative min-$p$ threshold of $\rho = e^{-13}$ (approx. $2.3 \times 10^{-6}$). This threshold is selected to surgically prune only the numerical noise tail while preserving the semantic richness of the distribution. We monitor training stability through the KL Divergence between the inference and training policies. We assess generalization performance on the AIME25 benchmark, reporting avg@32 scores to ensure robust evaluation.

\begin{figure}[htbp]
    \centering
    \includegraphics[width=0.95\linewidth]{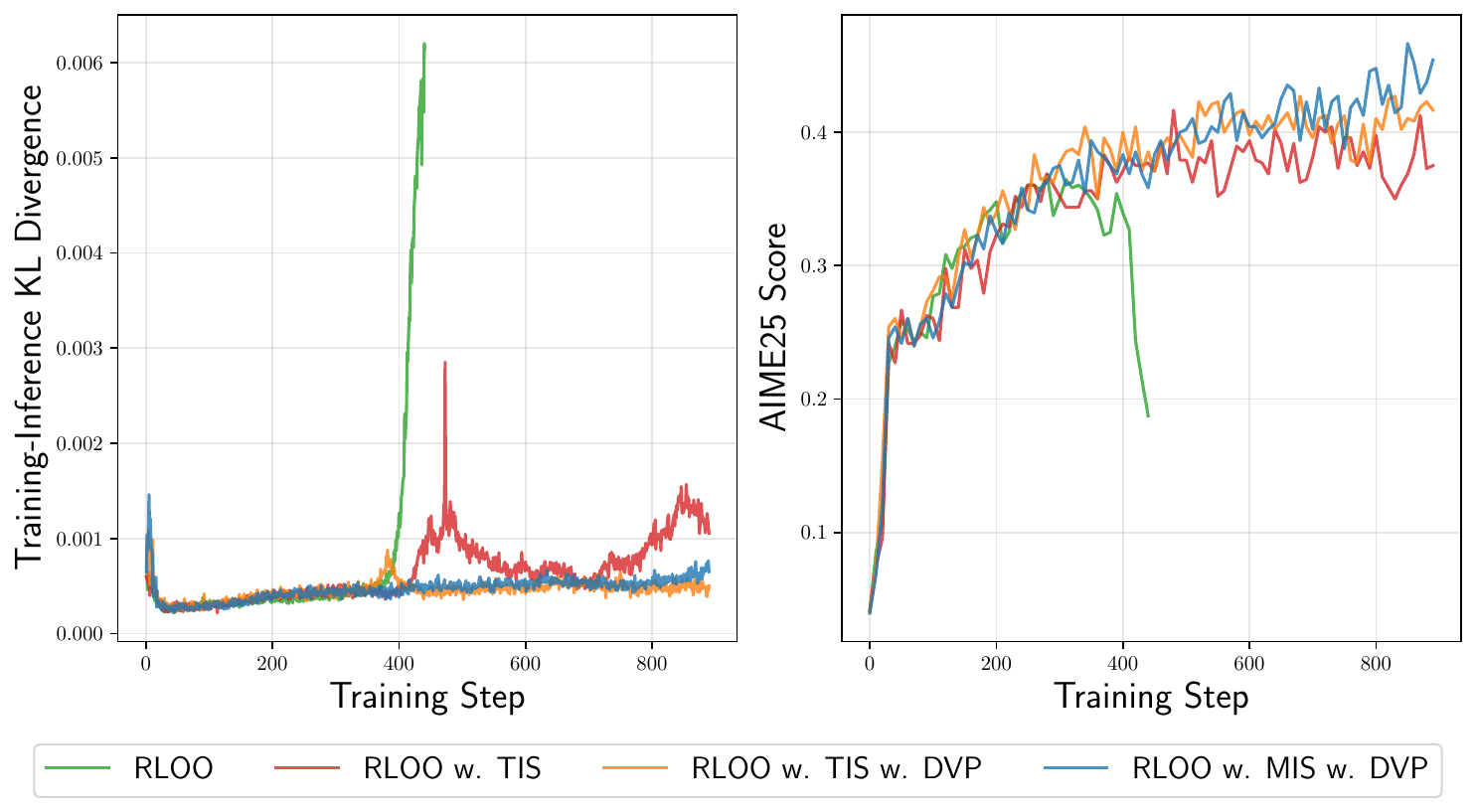}
    \caption{Comparison results on the Training-Inference KL Divergence and AIME25 Scores.}
    \label{fig:dvp}
\end{figure}

The empirical results, illustrated in Figure~\ref{fig:dvp}, provide strong evidence for the necessity of vocabulary control. Naive RLOO suffers from immediate and catastrophic training collapse, evidenced by the rapid explosion of the KL divergence in the early stages of training. This confirms that the training-inference mismatch is severe enough to derail learning when left unchecked. While TIS attempts to mitigate this instability by clipping extreme weights, it fails to address the root cause, resulting in persistent mismatch and suboptimal performance. In contrast, DVP effectively stabilizes the learning process. By excluding the numerically unstable tail from the objective, DVP maintains the stable KL divergence throughout the training trajectory. This numerical stability translates directly into performance gains: the model learns consistently without the degradation seen in compared baselines. Notably, the combination of MIS and DVP achieves the strongest results, yielding a 26.55\% improvement over the peak performance of naive RLOO. This suggests that DVP complements importance sampling techniques by removing the source of bias, allowing methods like MIS to more effectively manage the variance of the remaining valid tokens.

\section{Conclusion and Discussion}
\label{sec:conclusion}

In this work, we have established that training instability in LLM -RL is not merely a transient engineering artifact, but a structural inevitability arising from the divergence between high-throughput inference and high-precision training. Our theoretical diagnosis reveals the root cause: an asymmetric sensitivity to numerical noise. We prove that the log-probability mismatch bound scales with $(1-p)$, meaning that while high-probability tokens remain robust to aggressive inference optimizations, the low-probability ``tail'' becomes a region of unbounded divergence. Crucially, we demonstrated that errors in this regime are not zero-mean but systematically biased, accumulating over long reasoning chains to drive catastrophic training collapse.

To resolve this, we introduce {Dynamic Vocabulary Pruning (DVP)}. Rather than relying on reactive stabilization techniques—such as gradient clipping or weight truncation—which treat symptoms after they appear, DVP proactively redesigns the learning objective. By effectively performing ``surgery'' on the action space, DVP restricts gradient estimation to a dynamically defined safe support where the training and inference policies remain numerically consistent. Our empirical results on mathematical reasoning benchmarks confirm that this principled exclusion of the vocabulary tail is sufficient to recover stability, enabling the use of aggressive inference optimizations without sacrificing learning performance. In future work, we will explore adaptive thresholds ($\rho$) that scale dynamically with the model's entropy or the specific difficulty of the reasoning step, which could offer a finer control over the bias-variance tradeoff. 

DVP accounts for the learning objective under the numerical limitations of the underlying hardware. We trade a theoretical optimization bias for a massive reduction in practical variance. We believe that this class of methods—which align the mathematical requirements of training with the physical realities of inference—will be essential for the next generation of scalable, efficient, and robust AI systems.




\newpage






\nocite{langley00}

\bibliography{ref}
\bibliographystyle{icml2026}

\newpage
\appendix
\onecolumn

\section{Implementation Details}
\label{app:implementation}

We provide a PyTorch implementation that handles numerical stability issues common in low-precision training.

\begin{lstlisting}[language=Python, caption={PyTorch implementation of Min-$P$ Masking for DVP.}]
import math
import torch
import torch.nn.functional as F

def apply_minp_masking(logits: torch.Tensor, rho: float = math.exp(-13), 
                       mask_value: float = -50.0) -> torch.Tensor:
    """
    Applies Min-$P$ masking to logits for Dynamic Vocabulary Pruning.
    
    The safe set is defined as: V_S(s) = {a | pi(a|s) >= max(pi(s)) * rho}
    In log-space: logit(a) >= max(logit) + log(rho)
    
    Args:
        logits: Input logits tensor of shape (..., vocab_size).
        rho: The pruning threshold ratio (default: e^-13).
        mask_value: The value to assign to pruned logits. -50.0 is used 
                    instead of -inf to ensure stability in BF16/FP16.
    
    Returns:
        The masked logits tensor, preserving gradients for safe tokens.
    """
    # 1. Compute threshold in a detached context to treat the set V_S as fixed.
    with torch.no_grad():
        # Using max(dim=-1) is computationally efficient
        max_logits = logits.max(dim=-1, keepdim=True).values
        threshold = max_logits + math.log(rho)
        
        # Create boolean mask: True where logits should be PRUNED
        mask = logits < threshold

    # 2. Apply masking using torch.where to preserve the computation graph
    #    for unmasked elements.
    return torch.where(mask, torch.tensor(mask_value, dtype=logits.dtype, device=logits.device), logits)
\end{lstlisting}

To integrate DVP into existing RL pipelines (e.g., RLOO, PPO), we adhere to the following protocols:

\begin{enumerate}
    \item \textbf{Fixed Safe Set Approximation:} The threshold computation and mask generation occur strictly inside a \texttt{torch.no\_grad()} block. This ensures we do not attempt to backpropagate through the discrete decision of whether a token is included in the set, which aligns with our derivation treating $\mathcal{V}_S(s)$ as fixed.
    
    \item \textbf{Logit-Space Thresholding:} We compute the threshold as $z_{\max} + \log(\rho)$ rather than converting to probabilities. This avoids underflow issues for small probabilities and reduces computational overhead by skipping the expensive \texttt{exp()} and sum operations required for full softmax during mask creation.
    
    \item \textbf{Numerical Stability in BF16:} We use a finite negative number ($-50.0$) rather than $-\infty$ for masking. In `bfloat16`, $e^{-50} \approx 1.9 \times 10^{-22}$, which is effectively zero for softmax calculation but prevents \texttt{NaN} gradients that can arise from multiplying zero probabilities with infinite log-probs during the backward pass.
    
    \item \textbf{Importance Weight Computation:} When computing importance weights $\frac{\pi_{\text{train}}^{\text{mp}}}{\pi_{\text{infer}}^{\text{mp}}}$, we compute log-probabilities independently on both systems and exponentiate the difference. This is numerically safer than direct division.

    \item \textbf{Token Veto Mechanism:} We implement an additional safety filter for catastrophic divergence. If any token within a trajectory exhibits an importance weight $\frac{\pi_{\text{train}}^{\text{mp}}}{\pi_{\text{infer}}^{\text{mp}}} < 10^{-4}$, we identify the mismatch as irreparable and discard the entire sequence from the gradient update. To ensure a fair comparison, this veto mechanism is applied uniformly across all methods in our experiments, including the naive RLOO.
\end{enumerate}




\end{document}